\documentclass[10pt]{article} % For LaTeX2e
\usepackage[preprint]{tmlr}

\usepackage{amsmath,amsfonts,bm}

\def\eqref#1{equation~\ref{#1}}
\def\1{\bm{1}}

\DeclareMathAlphabet{\mathsfit}{\encodingdefault}{\sfdefault}{m}{sl}
\SetMathAlphabet{\mathsfit}{bold}{\encodingdefault}{\sfdefault}{bx}{n}

\usepackage{hyperref}
\usepackage{url}
\usepackage{graphicx}
\usepackage{amsthm}
\usepackage{algorithm}
\usepackage{algorithmicx}
\usepackage{algpseudocode}
\usepackage{subcaption}
\usepackage{bm}
\usepackage{booktabs}

\newtheorem{theorem}{Theorem}
\newtheorem{lemma}{Lemma}           % Optional: for lemmas
\newtheorem{definition}{Definition}

\title{Critically-Damped Higher-Order Langevin Dynamics for Generative Modeling}

\author{\name Benjamin Sterling \email benjamin.sterling@stonybrook.edu \\
      \addr Department of Applied Mathematics \& Statistics\\
      Stony Brook University
      \AND
\name Chad Gueli \email chadgueli21@gmail.com 
\AND
\name Yousef El-Laham \email yousef.ellaham@stonybrook.edu \\
      \addr Department of Electrical and Computer Engineering\\
      Stony Brook University
\AND      
      \name M\'onica Bugallo \email monica.bugallo@stonybrook.edu \\
      \addr Department of Electrical and Computer Engineering\\
      Stony Brook University
}

\def\month{MM}  % Insert correct month for camera-ready version
\def\year{YYYY} % Insert correct year for camera-ready version
\def\openreview{\url{https://openreview.net/forum?id=XXXX}} % Insert correct link to OpenReview for camera-ready version

\begin{document}

\maketitle

\begin{abstract}
Denoising diffusion probabilistic models (DDPMs) represent an entirely new class of generative AI methods that have yet to be fully explored. They use Langevin dynamics, represented as stochastic differential equations, to describe a process that transforms data into noise, the forward process, and a process that transforms noise into generated data, the reverse process. Many of these methods utilize auxiliary variables that formulate the data as a ``position'' variable, and the auxiliary variables are referred to as ``velocity'', ``acceleration'', etc. In this sense, it is possible to ``critically damp'' the dynamics. Critical damping has been successfully introduced in Critically-Damped Langevin Dynamics (CLD) and Critically-Damped Third-Order Langevin Dynamics (TOLD++), but has not yet been applied to dynamics of arbitrary order. The proposed methodology generalizes Higher-Order Langevin Dynamics (HOLD), a recent state-of-the-art diffusion method, by introducing the concept of critical damping from systems analysis. Similarly to TOLD++, our approach proposes an optimal set of hyperparameters in the $n$-dimensional case, where HOLD leaves these to be user defined. Additionally, the new method provides closed-form solutions for the mean and covariance of the forward process that greatly simplify its implementation. Experiments are performed on the CIFAR-10 and CelebA-HQ $256 \times 256$ datasets, and validated against the FID metric.
\end{abstract}

\section{Introduction}

Deep generative models are a popular class of methods used to generate samples from intractable, high-dimensional distributions. They do so by learning to transport samples from tractable and simple distributions into desired target distributions. The first popular generative model was the generative adversarial network (GAN) \cite{gans}, and later normalizing flows \cite{normalizingflowsoriginal} and the variational auto-encoder (VAE) \cite{vaes}. Today, the dominant methodology for image, video, and non-linguistic sample generation is the denoising diffusion probabilistic model (DDPM) \cite{diffusion2015, diffusiondenoising}. The work of \cite{difflangevinfirst} alternatively proposes framing diffusion as a score matching problem with Langevin dynamics, and \cite{diffusioncts} unites these two perspectives using score matching with stochastic differential equations (SDEs). Also, \cite{diffusioncts} defines a more general version of the Ornstein-Uhlenbeck SDE that they call the Variance-Preserving SDE (VPSDE). This SDE often serves as a benchmark for other proposed SDEs, as it is the most commonly used.

While diffusion models that use the VPSDE perform well in many contexts, this SDE adds Brownian motion directly on top of the data, resulting in coarse diffusion trajectories. Shortly after \cite{diffusioncts}, critically-damped Langevin dynamics (CLD) was proposed to improve upon the VPSDE by adding an auxiliary variable, referred to as velocity \cite{dockhorn2021score}. The addition of velocity smooths the data variable's convergence from the data distribution to latent distribution by preventing the Ornstein–Uhlenbeck process from acting directly on the data. Subsequently, third-order Langevin dynamics (TOLD) \cite{hold} was created to improve upon CLD by simplifying the forward and backward stochastic processes and adding a second auxiliary variable, acceleration, further smoothing the diffusion trajectories. 

Prior work \cite{sterling} examines TOLD and argues that it is optimal to choose parameters that critically damp the forward diffusion process. The concept of ``critical damping" is borrowed from classical systems theory to describe the family of optimal parameter choices for harmonic oscillators. Non-optimal parameter choices classically result in ``overshooting" or ``undershooting" the controls in a system of interest. Critically-damped third-order Langevin dynamics (TOLD++) \cite{sterling} advances TOLD by deriving a set of critically-damped third-order dynamics and proving that these dynamics result in faster convergence to a stationary distribution. Beyond optimality, critically damping the third-order system simplifies the TOLD algorithm, reducing computational cost. The invention of TOLD++ motivates the possibility of applying critical damping to arbitrarily higher orders of Langevin dynamics; this is the precise goal of this work. We propose critically-damped higher-order Langevin dynamics (HOLD++), and further prove that critical damping is optimal in the $n$th dimensional case. The contributions of this work are listed as follows:
\begin{itemize}
    \item Closed-form expressions are derived for the mean and covariance of the forward process, greatly simplifying the method's implementation.
    \item The HOLD++ algorithm is derived and presented for arbitrary order dynamics.
    \item Optimality of critical damping is proven.
    \item The parameter choice that results in a critically-damped forward process is derived.
    \item The method is validated on a synthetic dataset generated from a mixture of Gaussians, CIFAR-10, and CelebA-HQ $256 \times 256$ datasets.
\end{itemize}
The next section details relevant works. Section \ref{sec:simple} describes the most simple diffusion algorithm, while Section \ref{sec:methods} details how it is expanded to an arbitrary number of auxiliary variables. Section \ref{sec:optimality} proves the aforementioned optimality, and Section \ref{sec:criticaldamping} derives the critically-damped parameter choices in detail. Section \ref{sec:experiments} presents the experimental results, and Section \ref{sec:conclusion} is the conclusion.

\section{Related Works}

Here we provide a review of relevant works. Underdamped diffusion bridges \cite{blessing2025underdamped} have recently been proposed to accelerate the dynamics of both the forward and backwards processes by introducing control variables, parameterized by neural networks, that operate by minimizing a suitable divergence between the forward and backward processes. Their approach fundamentally differs from this work's as we only consider constantly parameterized stochastic processes, and instead consider optimality only with respect to the speed of convergence of the forward process. One key benefit of this approach is that there exists a closed-form solution for our optimal forward process. It is important to note that in reference to the works of \cite{Cheng2017UnderdampedLM,blessing2025underdamped}, that damping in this context is defined a bit differently: in these other works, Langevin dynamics without auxiliary variables are overdamped as they do not model system acceleration, and Langevin dynamics with one auxiliary variable are under damped.

The work of \cite{singhal2023where} pioneers a similar approach to this one with the use of solving differential equations for the mean and covariance of the forward process. However, this is very much a precursor to the HOLD methodology, as it does not consider specific forward processes, only the general strategy. A different perspective is taken by \cite{DBLP:conf/icml/SinghalGR24} that considers non-linear forward diffusion paths, but introducing these non-linear dynamics results in a loss of closed-form sampling distributions, requiring local Taylor series approximations.

Since its invention, higher-order Langevin dynamics have found many applications, including the tasks of voice generation \cite{shi2024langwave} and noisy image restoration \cite{shi2024noisy}. They have even been adapted to operate on manifolds \cite{riemanniandiffusion}, and Lie groups \cite{liediffusion}. The latter has recently found application in generative modeling of chemical structures \cite{cornet2025kinetic}.

The main difference between HOLD and HOLD++ is that HOLD allows for arbitrary selection of parameters for the forward process, but it also requires custom derivations for each chosen parameterization using Putzer's Method \cite{putzer}. This work argues that the parameterization derived in HOLD is overdamped, meaning that it exhibits suboptimal modes of convergence. Therefore, HOLD++ derives the optimal set of parameters from a damping perspective and additionally simplifies the computational burden of this method.

More generally, the focus of modern methods is to shorten and smooth the trajectories of samples from the latent space to sample space. JKO-iFlow is a continuous normalizing flow whose trajectories are simplified by including a Wasserstein regularization term \cite{jko}. Subspace Diffusion \cite{subspacediffusion} and Wavelet Diffusion \cite{guth2022wavelet} take a similar approach by shortening the structure of diffusion processes in certain dimensions by exploiting convenient representations with subspace or wavelet decompositions. From this perspective, it will be demonstrated that HOLD++ similarly shortens and smooths the diffusion trajectory for each higher-order that is applied. 

While JKO-iFlow explicitly regularizes against the Wasserstein distance, the rest of these methods provide some form of implicit regularization. This work's methodology allows for an arbitrary amount of regularization by using a model order of $n$ as this controls the degree to which the diffusion process is smoothed.

\section{The Simple Case, $n=1$}
\label{sec:simple}

For simplicity in presentation, we begin our discussion with the case of univariate dynamics ($n=1$, where $n$ represents the order of the dynamics). We consider a univariate Ornstein Uhlenbeck (OU) process as the forward process:
\begin{equation*}
    d\bm{x}_t = -\xi \bm{x}_t dt + \sqrt{2\xi L^{-1}} d\bm{w}_t
\end{equation*}
where $\xi, L^{-1} > 0$ are algorithmic parameters, and $\{\bm{w}_t\}_{t \geq 0}$ is a standard Brownian motion. The stochastic process is indexed by $t$ where $t \in [0,T]$. This is a special case of the variance preserving SDE (VP-SDE) \cite{diffusioncts}, where the variance schedule is constant. The mean and variance of this stochastic process are governed by the following linear differential equations \cite{sarkka2019applied}:
\begin{align*}
    \frac{d \bm{\mu}_t}{dt} = -\xi \bm{\mu}_t, \quad \frac{d v_t}{dt} = -2\xi v_t + 2\xi L^{-1},
\end{align*}
whose solutions are as follows:
\begin{align*}
    \bm{\mu}_t = \exp(-\xi t)\bm{x}_0, \quad v_t = \left(L^{-1} + (\sigma_0^2 - L^{-1})e^{-2\xi t} \right).
\end{align*} 
Given the solutions of $\bm{\mu}_t$ and $v_t$, we have a closed-form expression for the distribution of $p(\bm{x}_t \mid \bm{x}_0) = \mathcal{N}(\bm{\mu}_t, v_t \bm{I}_h)$ where $\sigma_0^2>0$ is the initial variance of the process. The backward SDE is given by \cite{anderson1982reverse}:
\begin{equation}
    d\bm{x}_t = \left(\xi \bm{x}_t + 2\xi L^{-1} \nabla \log q_t(\bm{x}_t)\right)dt + \sqrt{2\xi L^{-1}}d\Bar{\bm{w}}_t,
    \label{eq:simplebackwardSDE}
\end{equation}
where $q_t$ is the probability density function (pdf) of the backward process at time $t$, $\nabla \log q_t$ is its corresponding score, and $\Bar{\bm{w}}_t$ is another Brownian motion. In practice, $\nabla \log q_t(\bm{x}_t, t)$ does not have a closed-form, and thus is typically approximated by a neural network $s_\theta(\bm{x}_t, t)$, parameterized by $\theta \in \Theta$. The parameters $\theta$ are optimized such that they minimize the following score-matching objective:
\begin{align*}
    \mathcal{L}(\theta) &= 2\xi L^{-1}\mathbb{E}_{t \in \mathcal{U}[0,T], \bm{x}_t \sim p(\bm{x}_t \mid \bm{x}_0)} \left(\left\lVert \nabla_{\bm{x}_t}\log p(\bm{x}_t \mid \bm{x}_0) - \nabla_{\bm{x}_t}\log q(\bm{x}_t, t) \right\rVert^2_2 \right).
\end{align*}
One may observe that sampling from $p(\bm{x}_t \mid \bm{x}_0)$ may be accomplished by performing the following for $\bm{\epsilon}_t \sim \mathcal{N}(\bm{0}, \bm{I}_h)$:
\begin{equation*}
    \bm{x}_t = e^{-\xi t}\bm{x}_0 + l_t \bm{\epsilon}_t,
\end{equation*}
where $l_t = \sqrt{\left(L^{-1} + (\sigma_0^2 - L^{-1})e^{-2\xi t} \right)}$. For a simple Gaussian, $\log p(\bm{x}_t \mid \bm{x}_0) \propto -\frac{1}{2l_t^2} \left\lVert \bm{x}_t - e^{-\xi t}\bm{x}_0 \right \rVert^2$, thus:
\begin{align*}
\nabla_{\bm{x}_t}\log p(\bm{x}_t \mid \bm{x}_0) &= \frac{-1}{l_t^2}\left(\bm{x}_t - e^{-\xi t}\bm{x}_0 \right) = \frac{-1}{l_t}\bm{\epsilon}_t.
\end{align*}
Our goal is to estimate the score of the backward process $\nabla_{\bm{x}_t}\log q(\bm{x}_t, t)$ with a neural network $s_\theta(\bm{x}_t, t)$, therefore upon substitution of the aforementioned quantities, the training loss becomes:
\begin{align*}
    \mathcal{L}(\theta) &= 2\xi L^{-1} \mathbb{E}_{t \in \mathcal{U}[0,T], \bm{x}_t \sim p(\bm{x}_t \mid \bm{x}_0)} \left\lVert \frac{1}{l_t}\bm{\epsilon}_t + s_\theta(\bm{x}_t, t) \right\rVert^2_2 \propto \mathbb{E}_{t \in \mathcal{U}[0,T], \bm{x}_t \sim p(\bm{x}_t \mid \bm{x}_0)} \left\lVert\bm{\epsilon}_t + l_t s_\theta(\bm{x}_t, t) \right\rVert^2_2.
\end{align*}

The final term represents the loss function used in this work. Sampling from the generative model is accomplished by numerically solving the learned backward SDE in Equation \ref{eq:simplebackwardSDE}, typically by using the Euler-Maruyama approximation.

\section{Methods}
\label{sec:methods}
 
 To define the order $n>1$ forward stochastic process $p(\bm{x}_t \mid \bm{x}_0)$ with coefficients $\gamma_1,\gamma_2,\ldots, \gamma_{n-1}, \xi\in\mathbb{R}$, we define the following variables:
\begin{align*}
    \bm{F} = \sum_{i=1}^{n-1}\gamma_i\left(\bm{E}_{i, i+1}-\bm{E}_{i+1,i}\right) 
             - \xi\bm{E}_{n,n}, \quad \bm{G} = \sqrt{2\xi L^{-1}}\bm{E}_{n,n}
\end{align*}
where $\bm{E}_{i, j}\in\mathbb{R}^{n\times n}$ is the matrix of all zeros with a one at index $i,j$.
For example, in the third-order setting,
\begin{align*}
\bm{F} =
\begin{bmatrix}
     0 & \gamma_1 & 0  \\
     -\gamma_1 & 0 & \gamma_2  \\
     0 & -\gamma_2 & -\xi \\
\end{bmatrix}, \quad
\bm{G}  = \begin{bmatrix}
     0 & 0 & 0 \\
     0 & 0 & 0 \\
     0 & 0 & \sqrt{2\xi L^{-1}} \\
\end{bmatrix}.
\end{align*}
Then, the process evolves independently in $h$ dimensions according to the SDE
\begin{equation}
    d\bm{x}_t = \mathcal{F} \bm{x}_t dt + \mathcal{G} d\bm{w},
    \label{eq:forward_sde}
\end{equation}
where $\mathcal{F} = \bm{F} \otimes \bm{I}_h$, and $\mathcal{G} = \bm{G} \otimes \bm{I}_h$. In Section \ref{sec:criticaldamping}, we will derive how to choose $\gamma_1,\gamma_2,\ldots,\gamma_{n-1}, \xi$ such that the diffusion process is critically damped. A critically damped diffusion process is defined below.

\begin{definition}
    Given a stochastic process $d\bm{x}_t = \mathcal{F}\bm{x}_tdt + \mathcal{G}d\bm{w}$, the process $\bm{x}_t$ is critically damped iff $\mathcal{F}$ possesses a single geometric eigenvalue.
\end{definition}

If the matrix $\mathcal{F}$ contains a single geometric eigenvalue, then there is a single eigenvalue that drives the system's evolution. We will demonstrate later that this results in a single mode of convergence, which turns out to be optimal.

The current problem of interest is to sample from the forward distribution governed by (\ref{eq:forward_sde}).
By construction, $p(\bm{x}_t \mid \bm{x}_0)$ is a Gaussian process, so the mean $\bm{\mu}_t$ and covariance $\bm{\Sigma}_t$ are governed by \cite{sarkka2019applied}
\begin{align}
    \frac{d \bm{\mu}_t}{dt} &= \mathcal{F}\bm{\mu}_t,\label{eq:mean_eq}   \\
    \frac{d \bm{\Sigma}_t}{dt} &= \mathcal{F}\bm{\Sigma}_t + (\mathcal{F}\bm{\Sigma}_t)^T + \mathcal{G}\mathcal{G}^T.
    \label{eq:var_eq}
\end{align}
The first contribution of this work is the simplification of the mean and covariance into analytical expressions. We formally state the result in the following theorem.

\begin{theorem}\label{thm:covformula}
The solutions to differential equations \eqref{eq:mean_eq} and \eqref{eq:var_eq} are
\begin{align}
    \bm{\mu}_t &= \exp(\mathcal{F}t)\bm{x}_0,\label{eq:mean_solution}   \\
    \bm{\Sigma}_t &= L^{-1}\bm{I} + \exp(\mathcal{F}t)\left( \bm{\Sigma}_0 - L^{-1} \bm{I}\right)\exp(\mathcal{F}t)^T,\label{eq:cov_solution}
\end{align}
where $\exp(\cdot)$ in this context refers to a matrix exponential given by $\exp(\cdot) = \sum_{k=0}^\infty \frac{1}{k!}(\cdot)^k$.
\end{theorem}

\begin{proof} The expression for $\bm{\mu}_t$ is the canonical solution to multivariate linear differential equations. To evaluate the covariance differential equation, note that 
\begin{align*}
    \mathcal{G}\mathcal{G}^T = (2\xi L^{-1}\bm{E}_{n,n})\otimes \bm{I}_h = -L^{-1}\left(\mathcal{F} + \mathcal{F}^T\right).
\end{align*}
If we define $\bm{C}_t = \boldsymbol{\Sigma}_t - L^{-1}\bm{I} = \bm{C}_t^T$, then the differential equation from Equation \ref{eq:var_eq} is expressible as
\begin{align*}
    \frac{d}{dt}\left( \bm{\Sigma}_t - L^{-1}\bm{I} \right) &= \mathcal{F}\left( \bm{\Sigma}_t - L^{-1}\bm{I} \right) + \left( \bm{\Sigma}_t - L^{-1}\bm{I} \right) \mathcal{F}^T \\
    \frac{d}{dt}\bm{C}_t &= \mathcal{F}\bm{C}_t + \bm{C}_t\mathcal{F}^T.
\end{align*}
We can immediately recognize this form as the continuous-time Lyapunov Equation with respect to $\bm{C}_t$. It is well known that this differential equation has the following solution:
\begin{align*}
    \bm{C}_t = \exp(\mathcal{F}t)\bm{C}_0\exp(\mathcal{F}t)^T.
    %\bm{\Sigma}_t - L^{-1}\bm{I} =  \exp(\mathcal{F}t)\left( \bm{\Sigma}_0 - L^{-1} \bm{I}\right)\exp(\mathcal{F}t)^T.
\end{align*}
By substitution, the solution to $\boldsymbol{\Sigma}_t$ readily follows.
\end{proof}

It is now trivial to evaluate the asymptotic distribution $\lim_{t \to \infty} p_t(\bm{x}) = p_{\infty}(\bm{x})$.

\begin{lemma}\label{lem:asymptoticcov}
If $\bm{F}$ possesses all negative eigenvalues, then $\lim_{t \to \infty} \bm{\mu}_t = \bm{0}$, and $\lim_{t \to \infty} \bm{\Sigma}_t = L^{-1}\bm{I}$; implying the following convergence in distribution $p_{\infty}$:
\begin{align*}
    \bm{x}_t \xrightarrow{d} \mathcal{N}(\bm{0}, L^{-1}\bm{I}).
\end{align*}
\end{lemma}

\begin{proof}
Since $\bm{F}$ is Hurwitz, $\lim_{t \to \infty} \exp(\bm{F}t) = \bm{0}$.
By Theorem \ref{thm:covformula}, this implies
\begin{align*}
    \lim_{t \to \infty} \bm{\mu}_t
        &= \lim_{t \to \infty}\left( \exp(\bm{F} t) \otimes \bm{I}_h\right) \bm{x}_0
        = \bm{0},   \\
    \lim_{t \to \infty} \bm{\Sigma}_t
        &= \lim_{t \to \infty} \Big[
            L^{-1}\bm{I} + \exp(\bm{F}t) \left( \Sigma_0 - L^{-1} \bm{I}\right)\exp(\bm{F}t)^T \Big] \otimes \bm{I}_h
        = L^{-1}\bm{I}.
\end{align*}
\end{proof}
Lemma \ref{lem:asymptoticcov} is simple yet consequential. Upon executing the reverse process, one must start by sampling the asymptotic distribution $\mathcal{N}(\bm{0}, L^{-1}\bm{I})$. The computations in Theorem \ref{thm:covformula} rely on $\exp(\mathcal{F} t) = \exp(\bm{F}t) \otimes \bm{I}_h$, which in this case of a single geometric eigenvalue, may be calculated as
\begin{equation}
    \exp(\bm{F}t) = \exp(\lambda_* t)\sum_{k=0}^{n-1}\frac{(\bm{F} - \lambda_*\bm{I})^k t^k}{k!}.
    \label{eq:expformula}
\end{equation}
In the case of HOLD \cite{hold}, which does not critically damp the forward dynamics, Putzer's spectral formula was used to calculate the matrix exponential \cite{putzer}. Critical damping provides an analytically and computationally simpler procedure. The HOLD and HOLD++ algorithms are given in Algorithm \ref{alg:HOLD}. These methods are functionally equivalent, differing only by how $\exp(\bm{F}t)$ is calculated. It is important to note that the space complexity (e.g., memory requirements) of the algorithm is $\mathcal{O}(nh)$, i.e., it depends linearly on the order of the dynamics.

\begin{algorithm}
\caption{HOLD++ Training Algorithm}\label{alg:HOLD}
\begin{algorithmic}[1]
    \State \textbf{Input:} Data $\bm{q}_0$, $\Sigma_0$, and Score Network $\mathfrak{S}$.
    \State Calculate $\mathcal{F} = \bm{F} \otimes \bm{I}_h$ \textcolor{blue}{using Theorem \ref{thm:cd}}
    \For {$k = 1$ to $n_{train}$}
        \State $\bm{p}^{(1)}, \bm{p}^{(2)}, \ldots \bm{p}^{(n-1)} \sim \mathcal{N}(\bm{0}, \alpha L^{-1}\bm{I}_h)$
        %\State $\bm{x}_0 \gets \begin{pmatrix}
        %    (\bm{q}_0)^T, (\bm{p}^{(2)})^T, (\bm{p}^{(3)})^T, \ldots , (\bm{p}^{(n)})^T
        %\end{pmatrix}^T$
        \State $\bm{x}_0 \gets \operatorname{vec}(\bm{q}_0, \bm{p}^{(1)}, \bm{p}^{(2)}, \ldots , \bm{p}^{(n-1)})$
        \State $t \gets \mathcal{U}(0, T)$
        \State Calculate $\exp(\mathcal{F}t)$ \textcolor{blue}{using (\ref{eq:expformula})}
        \State Calculate $\bm{\mu}_t, \bm{\Sigma}_t$ using \textcolor{blue}{Theorem \ref{thm:covformula}}
        \State Take $\bm{L}_t$, the Cholesky Decomposition of $\bm{\Sigma}_t$
        \State $\bm{\epsilon}_t^{(1)}, \bm{\epsilon}_t^{(2)},\ldots, \bm{\epsilon}_t^{(n)} \sim \mathcal{N}(\bm{0}, \bm{I}_h)$
        \vspace*{0.5em}
        \State $\bm{\epsilon}_t^{\text{full}} \gets \operatorname{vec}(\bm{\epsilon}_t^{(1)}, \bm{\epsilon}_t^{(2)}, \ldots \bm{\epsilon}_t^{(n)})$
        \vspace*{0.5em}
        \State $\bm{x}_t \gets \bm{\mu}_t + \bm{L}_t\bm{\epsilon}_t^{\text{full}}$
    \vspace*{0.5em}
    \State $\bm{s}_{\theta} \gets \mathfrak{S}(\bm{x}_t, t)$ \Comment{$\theta$ are score network parameters}
    \vspace*{0.5em}
    \State $\mathcal{L}(\theta) \gets \left\lVert \bm{\epsilon}_t^{(n)} + \bm{s}_{\theta} \bm{L}_t[nh,nh]  \right\rVert^2$
    \State \textbf{Backpropagate} through $\mathcal{L}(\theta)$  
    \EndFor
    \State \textcolor{blue}{\textit{Blue highlights this work's novel contributions.}}
    \State $\operatorname{vec}(\cdot)$ \textit{vertically stacks input vectors.}
\end{algorithmic}
\end{algorithm}
It is proven in Theorem \ref{thm:optimality} that the HOLD++ parameter choices, $\gamma_1, \gamma_2, \ldots, \gamma_{n-1}$, for fixed $\xi$, are optimal with respect to the objective function $\min_{\gamma_1, \gamma_2, \gamma_3, \ldots, \gamma_{n-1}} \max \bigg(\operatorname{Re} \left(\operatorname{eig}(\bm{F}) \right) \bigg)$. A brief discussion is provided in the next section that justifies this choice of objective. This result generalizes the optimality known for second-order dynamics and proven for third-order dynamics in \cite{sterling}.

\section{Optimality of Critical Damping}
\label{sec:optimality}

One may naturally question where the choice of objective this work proposes comes from. Theorem \ref{thm:covformula} and Lemma \ref{lem:asymptoticcov} demonstrate that convergence is achieved as $t \to \infty$, but are determined by the eigenvalues of $\exp(\bm{F}t)$ and therefore the eigenvalues of $\bm{F}$. If $\bm{F}$ is not critically damped and possesses an eigendecomposition $\bm{F} = \bm{Q} \boldsymbol{\Lambda} \bm{Q}^{-1}$, $\boldsymbol{\Lambda} = \operatorname{diag}(\lambda_1, \lambda_2, \ldots, \lambda_n)$, then one may classically derive $\exp(\bm{F}t)$, using the definition of $\exp$ presented in Theorem \ref{thm:covformula}, as:
\begin{align*}
    \exp(\bm{F}t) = \bm{Q} \exp(\bm{\Lambda}t) \bm{Q}^{-1} = \bm{Q} \operatorname{diag}(\exp(\lambda_1 t), \ldots \exp(\lambda_n t)) \bm{Q}^{-1}
\end{align*}

Supposing that all eigenvalues are real-valued (an equivalent argument could be made in the complex case, but would clutter notation making the discussion unnecessarily complicated), and without loss of generality that $\max(\lambda_1, \lambda_2, \ldots \lambda_n) = \lambda_k$, then as $t \to \infty$, the following approximation holds:
\begin{align*}
    \exp(\bm{F}t) \approx \bm{Q} \operatorname{diag}(0, 0, \ldots, \exp(\lambda_k t), 0, \ldots 0) \bm{Q}^{-1}.
\end{align*}
From a matrix perspective, the rate of convergence is completely governed by the maximum eigenvalue of $\bm{F}$. Therefore, the fastest convergence is achieved by minimizing $\lambda_k$, leading to the $\min \max$ eigenvalue objective presented in the following theorem.

\begin{theorem}
    If $\xi$ is fixed, then the critically damped parameter choices are optimal according to the following objective:
\begin{equation*}
\min_{\gamma_1, \gamma_2, \gamma_3, \ldots, \gamma_{n-1}} \max \bigg(\operatorname{Re} \left(\operatorname{eig}(\bm{F}) \right) \bigg).    
\end{equation*}
    \label{thm:optimality}
\end{theorem}
\begin{proof}
Recall the following identity for matrix trace in terms of the eigenvalues $\lambda_i$ of $\bm{F}$:
\begin{equation*}
    Tr(\bm{F}) = \sum_{i=1}^n \lambda_i = -\xi \quad \quad \to \quad \quad \xi = -\sum_{i=1}^n \lambda_i.
\end{equation*}
Consider only the real parts of each eigenvalue, as complex eigenvalues must come in conjugate pairs and cancel out as $\xi$ is real valued. Optimality under $\operatorname{min}_{\gamma_1,\gamma_2,\ldots,\gamma_{n-1}} \operatorname{\max}_i \lambda_i$ occurs when $\lambda_1 = \lambda_2 = \ldots = \lambda_n = -\frac{\xi}{n}$; if any eigenvalue is less than $-\frac{\xi}{n}$, then at least one other must be greater to conserve the sum, resulting in a suboptimal objective.
\end{proof}

There remains an important caveat to Theorem \ref{thm:optimality}. Existing works that use auxiliary variables prefer to define the first auxiliary variable, denoted $\bm{p}^{(1)}_t$, to be the time derivative of the variable data $\bm{q}_t$. A closer look at the forward equations reveal that $\frac{d}{dt}\bm{q}_t = \bm{p}^{(1)}_t$ upon setting $\gamma_1 = 1$. This is important because the choice of critically damped $\bm{F}$ is not unique. Supposing $\bm{F}$ is critically damped, take $\bm{F}' = \upsilon \bm{F}, \upsilon > 0, \mathcal{F}' = \bm{F}' \otimes \bm{I}_h$, then
\begin{align*}
    d\bm{x}_t = \mathcal{F} \bm{x}_t dt + \mathcal{G} d\bm{w} \quad \text{and} \quad d\bm{x}_t = \mathcal{F}' \bm{x}_t dt + \mathcal{G} d\bm{w},
\end{align*}
are both valid critically damped dynamics. In order to provide a single closed form solution and to compare HOLD++ to HOLD fairly, the design choice is taken so that $\gamma_1 = 1$. The previous proof does not assume $\gamma_1=1$ and therefore is not the same as being optimal under the objective $\min_{\gamma_2, \gamma_3, \ldots, \gamma_{n-1}, \xi} \max \left(\operatorname{Re} \left(\operatorname{eig}(\bm{F}) \right) \right)$. However, the $\bm{F}$ matrix generated according to the objectives $\min_{\gamma_2, \gamma_3, \ldots, \gamma_{n-1}, \xi} \max \left(\operatorname{Re} \left(\operatorname{eig}(\bm{F}) \right) \right)$ and $\min_{\gamma_1, \gamma_2, \ldots, \gamma_{n-1}} \max \left(\operatorname{Re} \left(\operatorname{eig}(\bm{F}) \right) \right)$, only differ by a scaling factor. This means that the behavior of the algorithm remains exactly the same, while $\gamma_1=1$ is a more convenient design choice as it is easier to fix this than some arbitrary value of $\xi$. Therefore $\gamma_1$ is fixed instead of $\xi$.

\section{Critical Damping}
\label{sec:criticaldamping}

In the following theorem, analytical expressions for $\gamma_2, \gamma_3, \ldots, \gamma_{n-1}, \xi$ for critically-damped Langevin dynamics are derived.

\begin{theorem}\label{thm:cd}
    For $n>1$, define
    $\bm{F} = \sum_{i=1}^{n-1}\gamma_i\left(\bm{E}_{i, i+1}-\bm{E}_{i+1,i}\right) -\xi\bm{E}_{n,n}$, as in Section \ref{sec:methods}.
    If
    \begin{align*}
        \gamma_{n-i}&= \left\lvert \lambda_* \right\rvert\sqrt{\frac{n^2-i^2}{4i^2-1}}, &    &\text{and} & \xi&= n\left\lvert\lambda_*\right\rvert,
    \end{align*}
    for $i\in\{1,2,\ldots,n-1\}$, then $\bm{F}$ has a single geometric eigenvalue $\lambda_*$, making $\bm{F}$ critically damped.
    By convention, $\gamma_{1}$ is set to $1$, in this case $\lambda_*=-\sqrt{2n-3}$.
\end{theorem}
The remainder of this section is dedicated to proving this theorem. \footnote{For brevity, we allow a slight abuse of notation when setting $\gamma_{n-i}$. In the chosen notation, the $n$ in the formula is implied by the $n$ before the minus sign in the index. For example, $\gamma_{n-2}$ is the $(n-2)$th value in the model of order $n$ while $\gamma_{(n-1)-1}$ is the $(n-2)$th value in the model of order $n-1$, and in general, $\gamma_{n-2}\ne\gamma_{(n-1)-1}$. Since $n$ is fixed throughout the theory of this work, we further refine the notation and understand $\gamma_i=\gamma_{n-(n-i)}$.}

\subsection{From Matrices to Polynomials}
Establishing Theorem \ref{thm:cd} requires several intermediate steps; the first is to express the characteristic polynomial of $\bm{F}$.
For $j\in\{1, 2, \ldots, n\}$, define $d_j:\mathbb{R}\to\mathbb{R}$ to be the characteristic polynomial of the submatrix containing the first $j$ rows and first $j$ columns of $\bm{F}+\xi\bm{E}_{n,n}$ in order.
We work out an expression for $d_j$ in Lemma \ref{lem:dj}.
For now, it is easy to verify that taking a partial Laplace expansion gives
\begin{align}
    d_{j+1}(\lambda)=-\lambda d_j(\lambda)+\gamma_j^2d_{j-1}(\lambda).
    \label{eq:djdef}
\end{align}
Finally, define $q:\mathbb{R}\to\mathbb{R}$ to be the characteristic polynomial of $\bm{F}$.
Then,
\begin{align*}
    q(\lambda) = \mathrm{char}\left[\bm{F}+\xi\bm{E}_{n,n}-\xi\bm{E}_{n,n}\right](\lambda) = d_n(\lambda)-\xi d_{n-1}(\lambda),
\end{align*}
due to the multilinearity of the determinant.

\subsection{From Polynomials to Their Coefficients}
Now, to express $q$ as a closed-form in $\lambda$, we introduce the bivariate sequence
\begin{equation}
\label{eq:sjk}
s_{j,k} = 
\begin{cases} 
    \gamma_j^2 s_{j-2,k-1} + s_{j-1,k}, & \text{for } j > 2k-2, \\ 
    1, & \text{for } k = 0, \\
    0, & \text{otherwise}.
\end{cases}
\end{equation}
Then, we have the following expression for $d_j$.
\begin{lemma}
\label{lem:dj}
$d_j$ as defined in Equation \ref{eq:djdef}, and $s_{j,k}$ as defined in Equation \ref{eq:sjk} are related by
\begin{align*}
    d_j(\lambda)= (-1)^{j}\sum_{k=0}^{\lfloor j/2 \rfloor} \lambda^{j-2k}s_{j-1,k}.
\end{align*}
\end{lemma}
Lemma \ref{lem:dj} is proven in Appendix \ref{app:djproof}. Immediately, Lemma \ref{lem:dj} implies
\begin{align}
q(\lambda) &= (-1)^{n} \sum_{k=0}^{\lfloor n/2 \rfloor} \lambda^{n - 2k} s_{n-1,k} - (-1)^{n-1} \xi \sum_{k=0}^{\lfloor (n - 1)/2 \rfloor} \lambda^{n - (2k + 1)} s_{n-2,k}
\label{eq:qofs}
\end{align}

We next use the results of Lemma \ref{lem:aj}, established in Appendix \ref{app:ajproof}. Lemma \ref{lem:aj} can be restated as $r$ is a root of every non-constant derivative of $m$ if and only if $r$ is the unique root of $m$. Choosing $\lambda_*$ as our root and matching terms from Lemma \ref{lem:aj} with Equation \ref{eq:qofs}, it is easy to see that
\begin{align*}
    s_{n-1,k} &= \binom{n}{2k}\lambda_*^{2k}, \\
    \xi s_{n-2,k} &= \binom{n}{2k+1}(-\lambda_*)^{2k+1}.
\end{align*}
Further, because $s_{n-2,0}=1$ as defined in Equation \ref{eq:sjk}, we have $\xi = \xi s_{n-2,0}=-n\lambda_*;$ which allows rearranging to get
\begin{align*}
    s_{n-2,k}=\binom{n}{2k+1}\frac{\lambda_*^{2k}}{n}=\frac{\binom{k+1}{k}}{\binom{2(k+1)}{2k}}\binom{n-1}{2k}\lambda_*^{2k}.
\end{align*}
More broadly, we observe the pattern and propose the following general form for $s_{n-i,k}$:
\begin{align}\label{eq:snk}
    s_{n-i, k}=\frac{\binom{i+k-1}{k}}{\binom{2(i+k-1)}{2k}}\binom{n-i+1}{2k}\lambda_*^{2k}.
\end{align}
What we are really trying to do here is solve a difference equation proposed in Equation \ref{eq:sjk}. The sequence of $\gamma_i$ are uniquely determined from the sequence $s_{j,k}$, so if we solve for the resulting $\gamma_i$, then satisfy the recurrence relation in Equation \ref{eq:sjk}, then we have successfully solved the difference equation.

\begin{lemma} \label{lem:stog}
    Fix $n>0$. For $\{s_{j,k}\}$ as defined in Equation \ref{eq:sjk}, if for $n-i > 2k-2$,
    \begin{align*}
        s_{n-i, k}=\frac{\binom{i+k-1}{k}}{\binom{2(i+k-1)}{2k}}\binom{n-i+1}{2k}\lambda_*^{2k}, \quad \text{ then } \quad \gamma_{n-i}^2=\frac{n^2-i^2}{4i^2-1}\lambda_*^2
    \end{align*}

\end{lemma}

\begin{lemma} \label{lem:differenceeq}
    Fix $n>0$. The following $s_{n-i,k}$, for $n-i > 2k-2$, satisfies the definition from Equation \ref{eq:sjk}.
    \begin{align*}
        s_{n-i, k}=\frac{\binom{i+k-1}{k}}{\binom{2(i+k-1)}{2k}}\binom{n-i+1}{2k}\lambda_*^{2k}.
    \end{align*}
\end{lemma}

Lemma \ref{lem:stog} and Lemma \ref{lem:differenceeq} are proven in Appendices \ref{app:proofstog} and \ref{app:proofdifferenceeq} respectively. This leads us quite easily to the following lemma.

\begin{lemma} \label{lem:lambdastar}
    The choice of $\gamma_1 = 1$ implies that 
    \begin{align*}
        \lambda_* = -\sqrt{2n-3}.
    \end{align*}
\end{lemma}

\begin{proof}
Using the formula in Lemma \ref{lem:stog}
\begin{equation*}
    1 = \gamma_1^2 = \gamma_{n-(n-1)}^2 = \frac{n^2 - (n-1)^2}{4(n-1)^2-1}\lambda_*^2 = \frac{\lambda_*^2}{2n-3}.
\end{equation*}
This directly implies $\lambda_* = -\sqrt{2n-3}$, as $\lambda_*$ is designed to be negative.
    
\end{proof}

\section{Experiments}
\label{sec:experiments}
To verify the efficacy of HOLD++, the algorithm is demonstrated on the forward process of a 1D Gaussian Mixture Model in Figure \ref{fig:oneDdistribution}, and then run on the 2D Multi-modal Gaussian, CIFAR-10 \cite{cifar10}, and CelebA-HQ $256 \times 256$ \cite{CelebAHQ} datasets. The CIFAR-10 dataset was run over model orders 2 through 5. However, due to the computational costs associated with the CelebA-HQ $256 \times 256$ dataset, the best performing order on the CIFAR-10 dataset was chosen for this dataset; this happened to be $n=3$.

\begin{figure}
    \centering
    \includegraphics[width=0.5\linewidth]{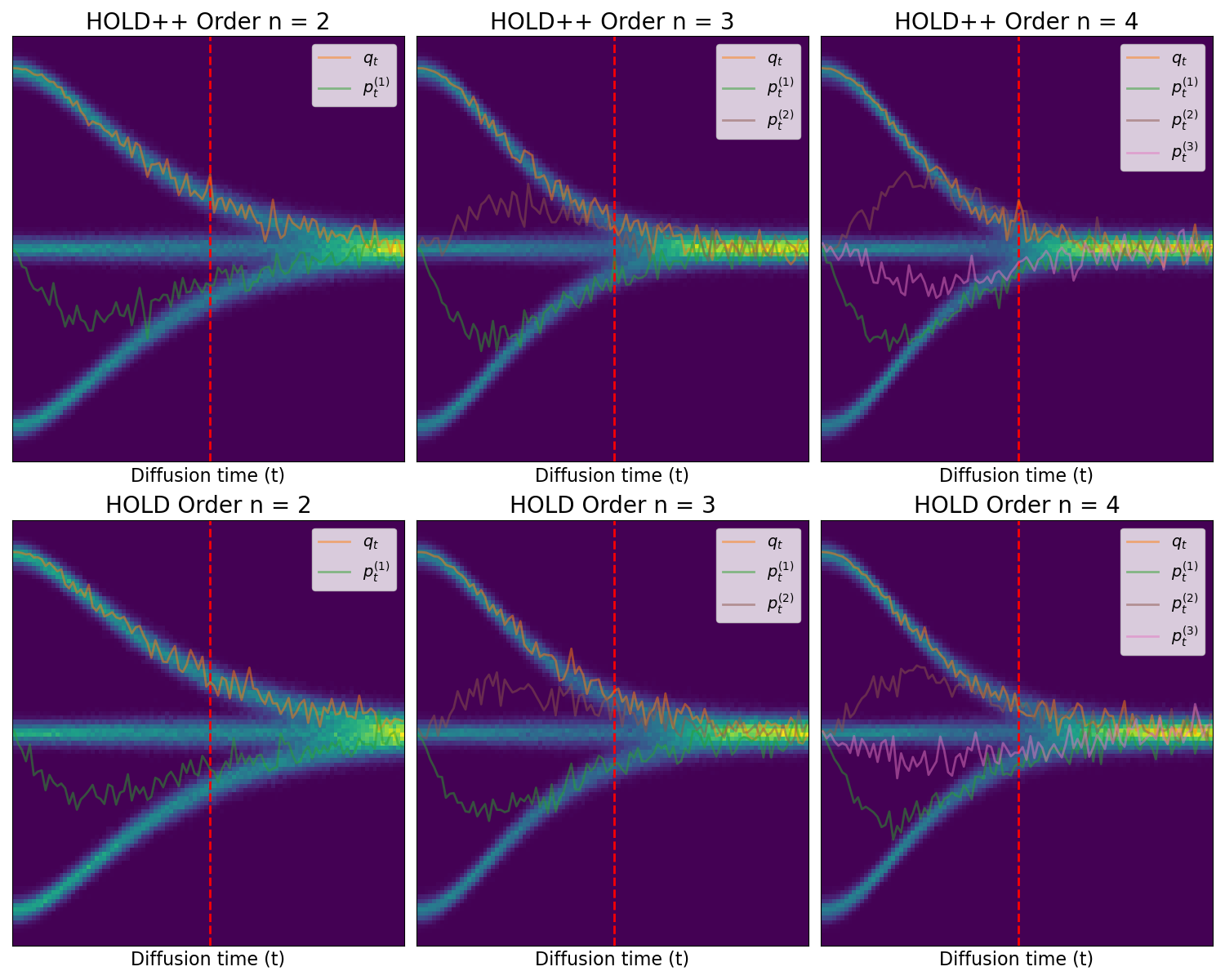}
    \caption{Diffusion paths of HOLD and HOLD++ on a Gaussian Mixture Model with three modalities. Data sample trajectories are denoted $\bm{q}_t$, and the auxiliary variables are $\bm{p}^{(1)}_t, \bm{p}^{(2)}_t, \ldots \bm{p}^{(n-1)}_t$. Together, $\bm{x}_t = \operatorname{vec}(\bm{q}_t, \bm{p}^{(1)}_t, \bm{p}^{(2)}_t, \ldots \bm{p}^{(n-1)}_t)$. Observe that for each order, HOLD++ converges faster to a normal distribution than HOLD does.}
    \label{fig:oneDdistribution}
\end{figure}

\subsection{Multi-modal Gaussian toy datasets}

The concept of comparing HOLD and HOLD++ is easily understood from Figure \ref{fig:oneDdistribution}. Increasing the model order adds more auxiliary variables to the forward diffusion process, thereby smoothing the transformations of 1D Gaussian Mixture Models into normal distributions. The dashed vertical red lines annotate the same diffusion times, making it easier to compare convergences between HOLD and HOLD++. By visual inspection, one may observe that each HOLD++ distribution converges faster to normal than each respective HOLD distribution.

The rest of this subsection presents a forward process of the Multi-modal Gaussian toy dataset in Figure \ref{fig:forwardgaussian}, and comparative validation losses in Figures \ref{fig:Wasserstein3} and \ref{fig:Wasserstein4}. Figure \ref{fig:forwardgaussian} used the following parameters to emphasize a more visible convergence: $n=4, L^{-1} = 0.2, \alpha=0.1, T=3.0$. Referring to Figure \ref{fig:forwardgaussian}, the more efficient diffusion process is the one whose distribution converges to $\mathcal{N}(\bm{0}, L^{-1}\bm{I})$ the fastest. One may visually observe that HOLD++ converges faster, especially at times $t=1.8, t=2.4$, as the HOLD++ data points more closely resemble a 2D normal than their HOLD counterparts. To numerically test normality, each snapshot contains a Shapiro-Wilk statistic $W$. The Shapiro-Wilk statistics are calculated along the 1D marginal distributions, and the minimum is presented. $W$ closest to $1$ indicates that the samples are nearly normally distributed, whereas lower $W$ indicates that the samples are further away from being normally distributed. It is clear through the $W$'s that HOLD++ converges to a normal distribution faster than HOLD.

In Figures \ref{fig:Wasserstein3} and \ref{fig:Wasserstein4} competing generative models, one using the HOLD++ SDE and one using the HOLD SDE, are trained on the 2D multi-modal Gaussian dataset. The following diffusion parameters were used: $L^{-1} = 2.0, T = 2.0, \alpha = 0.1, n_{\text{steps}} = 50$. Furthermore, a batch size of 1024 was used on a Multilayer Perceptron (MLP) network with 9 hidden layers and 128 hidden units. The network was trained with a learning rate of $2\times 10^{-6}$ for 100000 iterations. The validation metric is calculated by computing the Wasserstein distance along each 1D marginal distribution between the training data and generated samples and reporting the average. The Wasserstein distance is used instead of the Shapiro-Wilk statistic because these figures are concerned with generated sample quality, not the samples' proximity to being normally distributed. Figure \ref{fig:Wasserstein3} perfectly illustrates the benefits of critical damping when $n=3$. The HOLD++ curve converges faster and then starts to slowly diverge, indicating overfitting. The HOLD curve converges throughout the entire window, but never reaches the minimum that HOLD++ does between 20{,}000 and 40{,}000 iterations. The same experiment is run in Figure \ref{fig:Wasserstein4}, except $n=4$. Increasing $n$ by $1$ only serves to increase the performance benefit of using HOLD++ instead of HOLD.

\begin{figure}
    \centering
    \includegraphics[width=1.0\linewidth]{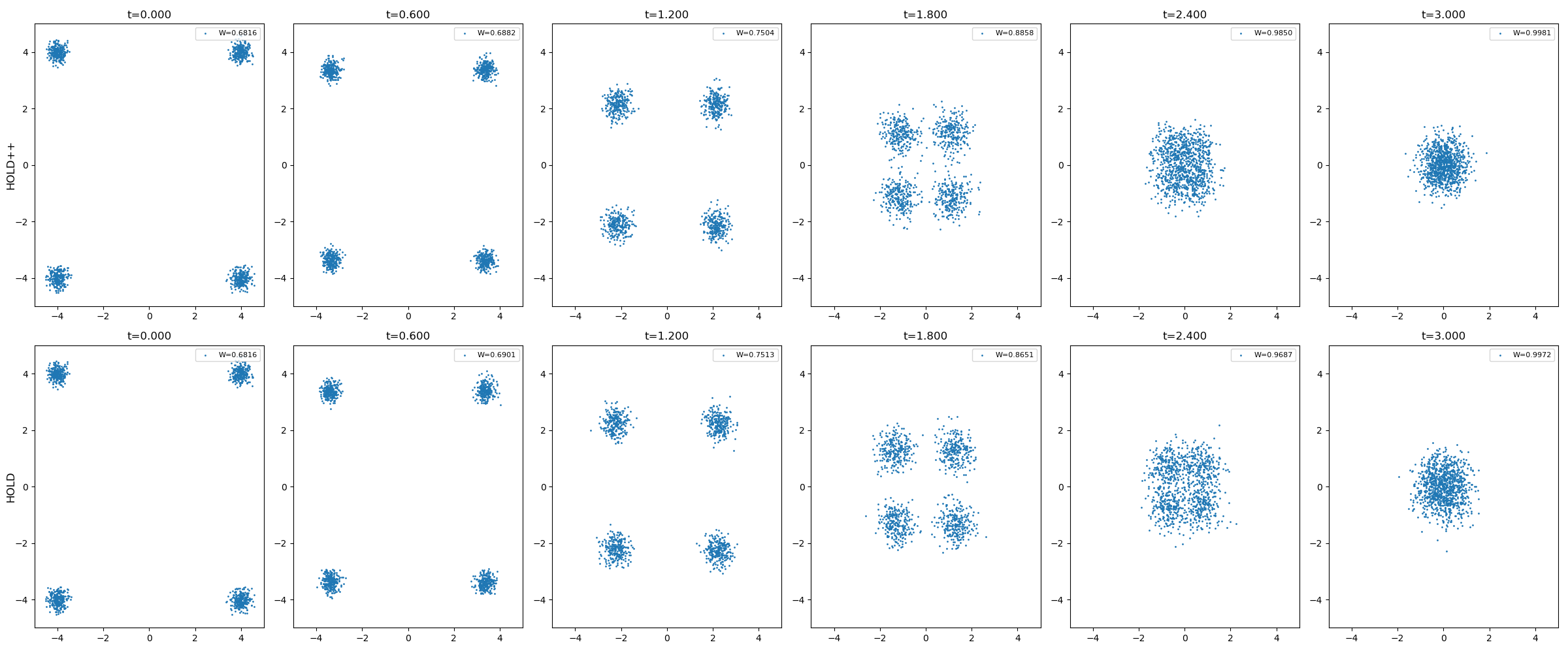}
    \caption{Forward Process of Multi-modal Gaussian toy dataset using the HOLD++ and HOLD SDEs, for $n=4$. $W$ denotes the minimum Shapiro-Wilk statistic applied to the 1D marginal distributions; $W$ closer to $1$ indicates the distribution is closer to being normally distributed. The distributions under HOLD++ converge faster to a normal distribution as demonstrated visually and measured by $W$.}
    \label{fig:forwardgaussian}
\end{figure}

\begin{figure}[ht]
    \centering
    \begin{subfigure}[b]{0.4\textwidth}
        \includegraphics[width=\linewidth]{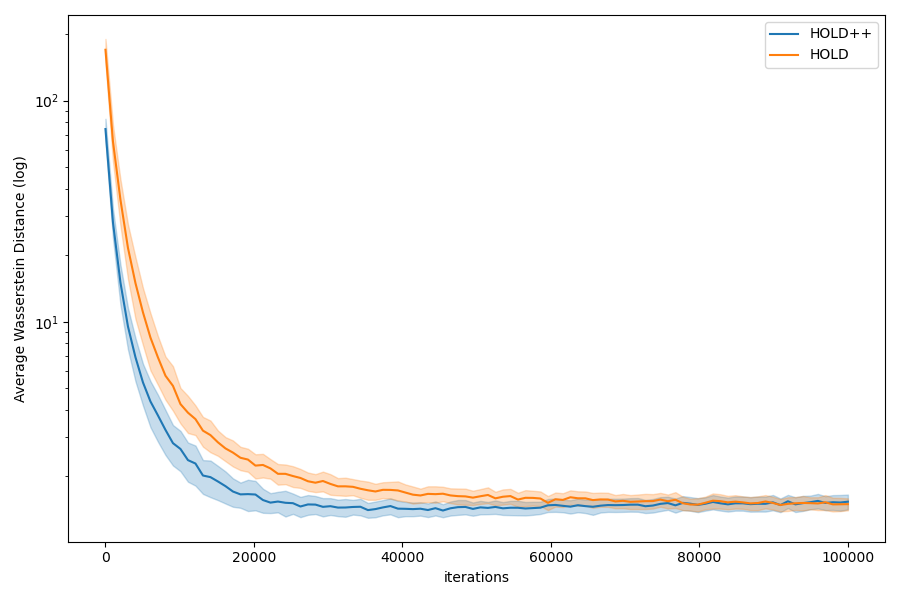}
        \caption{Average Wasserstein Distance for $n=3$ with 95\% confidence intervals.}
    \label{fig:Wasserstein3}
    \end{subfigure}
    \hfill
    \begin{subfigure}[b]{0.4\textwidth}
        \includegraphics[width=\linewidth]{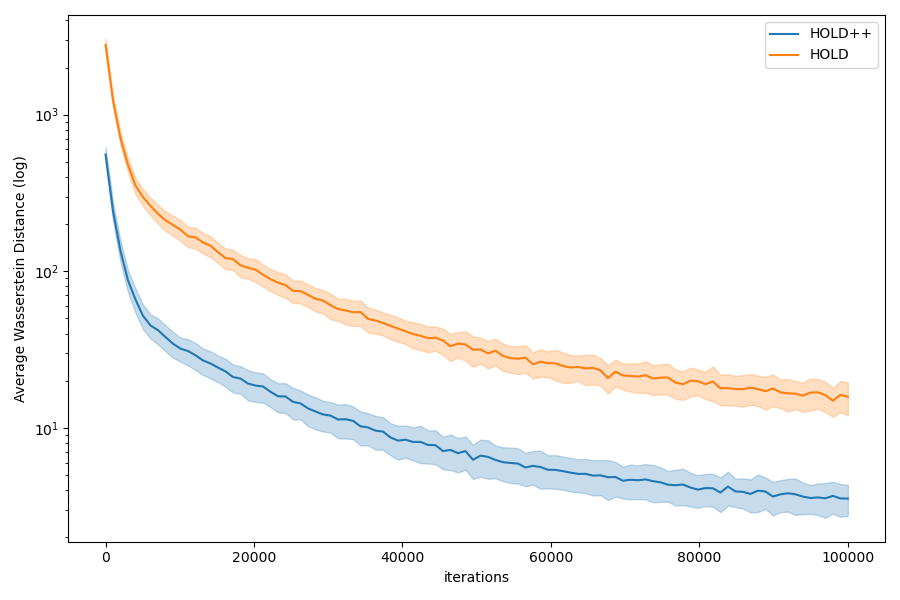}
        \caption{Average Wasserstein Distance for $n=4$ with 95\% confidence intervals.}
    \label{fig:Wasserstein4}
    \end{subfigure}
    \caption{Average Wasserstein Distances on the 2D multi-modal Gaussian dataset}
    \label{fig:Wasserstein}
\end{figure}

\subsection{CIFAR-10}

Critically-damped third-order Langevin dynamics is a special case of HOLD++ for order $n=3$. It was experimentally demonstrated in the third-order case that critically damping the forward dynamics results in faster FID convergence for each comparative training iteration; see \citep{sterling} for implementation details. The corresponding table is imported into Table \ref{tab:fid_comparison_cifar}. This demonstrates that HOLD++ has already been demonstrated to outperform state-of-the-art.

\begin{table}[t!]
\centering
\begin{tabular}{ccc}
\toprule
\textbf{Training Iterations} & \textbf{HOLD} (n=3) & \textbf{HOLD++} (n=3) \\
\midrule
$3 \times 10^5$ & 8.232 & \textbf{7.518} \\
$4 \times 10^5$ & 7.900 &  \textbf{7.180}\\
$5 \times 10^5$ & 7.648 &  \textbf{6.818}\\
$6 \times 10^5$ & 7.520 &  \textbf{6.704}\\
$7 \times 10^5$ & 7.475 &  \textbf{6.498}\\
\bottomrule
\end{tabular}
\caption{FID ($\downarrow$) Scores on CIFAR-10 Dataset imported from \citep{sterling}.}
\label{tab:fid_comparison_cifar}
\end{table}

\begin{figure}
    \centering
    \includegraphics[width=0.75\linewidth]{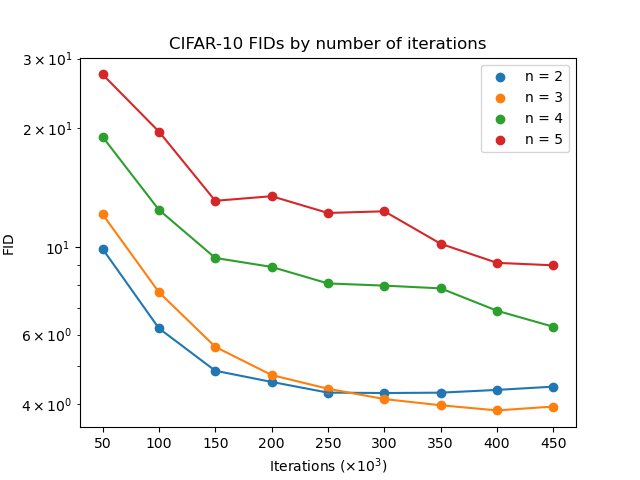}
    \caption{FID Scores of CIFAR-10 runs with varying dynamic orders.}
    \label{fig:cifarfids}
\end{figure}

Further experiments performed with the CIFAR-10 dataset utilized an NVIDIA A100 GPU and a training batch size of 128. Just like in the work of \cite{sterling}, a Noise Conditional Score Network++ (NCSN++) was used with 4 BigGAN type residual blocks and a DDPM attention module with resolution of 16. During inference time, the EM Method was used with 250 discrete steps and 50,000 samples, and with an evaluation batch size of 16. The following HOLD++ hyper-parameters were used: $T=5.0, L^{-1}=0.5, \alpha=0.08$. Figure \ref{fig:cifarfids} report the FID scores for each training session. Of all the tested model orders, order $n=3$ performed the best due to its lowest FID score. It is of note that model order $n=7$ was attempted, but failed to generate any meaningful images after $50{,}000$ training iterations, and thus deemed to fail for the chosen hyper-parameters. Model order $n=6$ is omitted as well because $50{,}000$ FID samples resulted in memory issues, and the model exhibited worse convergence than $n=5$. One shortcoming of this comparison method is that different model orders potentially need more iterations to converge, but we compare over the same number of training iterations. FIDs up to $450{,}000$ training iterations are provided to demonstrate that a stable FID level is reached. Samples are presented for the top two performing model orders in Figure \ref{fig:cifar10}. By visual inspection, both sets of images are of good quality.

\begin{figure}[ht]
    \centering
    \begin{subfigure}[b]{0.45\textwidth}
        \includegraphics[width=\textwidth]{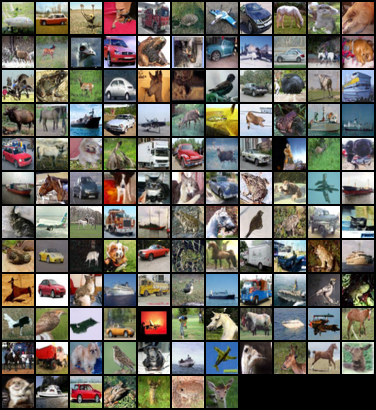}
        \caption{Model Order 2}
        \label{fig:cifarn2}
    \end{subfigure}
    \hfill
    \begin{subfigure}[b]{0.45\textwidth}
        \includegraphics[width=\textwidth]{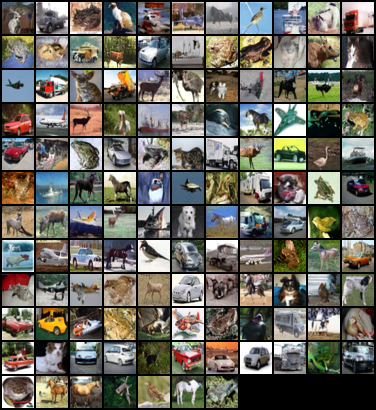}
        \caption{Model Order 3}
        \label{fig:cifarn3}
    \end{subfigure}
    \caption{CIFAR-10 images generated at $450{,}000$ training iterations for model orders 2 and 3.}
    \label{fig:cifar10}
\end{figure}

\subsection{CelebA-HQ $256\times 256$}

As mentioned previously, a single run with $n=3$ was performed on the CelebA-HQ $256\times 256$ dataset. The experiment utilized an NVIDIA H100 GPU and a training batch size of 16. The rest of the configuration remained the same as for the CIFAR-10 runs, except the learning rate was reduced to accommodate the smaller batch size. Samples without cherry picking are presented in Figure \ref{fig:celebahq}, and after $1{,}200{,}000$ training iterations, we obtain an FID of $17.07$. This FID is a bit high due to lack of refined training that stemmed from a lack of further compute time.

%17.066295

\section{Conclusion}
\label{sec:conclusion}
This manuscript builds a generalization of HOLD and TOLD++, HOLD++, that provides closed-form solutions to critically-damped higher order Langevin dynamics of arbitrary order. The closed-form solutions of the mean and covariance, and of the critically-damped parameters are rigorously derived, and it is further proven that critical-damping is optimal for the HOLD parameterization. Furthermore, the approach is validated on the CIFAR-10 and CelebA-HQ $256 \times 256$ datasets. Future work on this topic could entail research into other possible SDEs, and the desirable properties such different SDEs could provide the diffusion process.

\clearpage

%\subsubsection*{Broader Impact Statement}

%\subsubsection*{Author Contributions}
%Benjamin Sterling and Chad Gueli contributed equally to this work. 

\bibliography{refs}
\bibliographystyle{tmlr}

\appendix

\section{Proof of Lemma \ref{lem:dj}}
\label{app:djproof}
\begin{proof}
The proof proceeds by induction on $j$. The base case for $d_1$ is trivial, and because
\begin{align*}
    s_{1,1} &= \gamma_1^2 s_{-1,0} + s_{0,1} = \gamma_1^2, \\
    d_2(\lambda) &= \lambda^2+\gamma_1^2=(-1)^2\sum_{k=0}^{1} \lambda^{2-2k}s_{1,k}.
\end{align*}
Inductive Step:
\begin{align*}
    d_{n+1} &= -\lambda d_{n} + \gamma_{n-1}^2 d_{n-1} \\
    &=-\lambda(-1)^{n}\sum_{k=0}^{\lfloor n/2 \rfloor} \lambda^{n-2k}s_{n-1,k} \\
    &+ \gamma_{n-1}^2 (-1)^{n-1}\sum_{k=0}^{\lfloor (n-1)/2 \rfloor} \lambda^{n-2k-1}s_{n-2,k}.
\end{align*}
Now, we must break up cases into even and odd $n$:
\noindent Case $n$ is even:
When this happens, $\lfloor \frac{n}{2} \rfloor = \frac{n}{2}, \lfloor \frac{n-1}{2} \rfloor = \frac{n}{2}-1, \lfloor \frac{n+1}{2} \rfloor = \frac{n}{2},$
\begin{align*}
    d_{n+1} &= \lambda(-1)^{n-1}\sum_{k=0}^{n/2} \lambda^{n-2k}s_{n-1,k} \\
    &+ \gamma_{n-1}^2 (-1)^{n-1}\sum_{k=0}^{ n/2-1 } \lambda^{n-2k-1}s_{n-2,k}.
\end{align*}
Then by reindexing the second sum and noting $(-1)^{n-1} = (-1)^{n+1}$:
\begin{align*}
    d_{n+1} &= (-1)^{n+1}\sum_{k=0}^{n/2} \lambda^{n-2k+1}s_{n-1,k} \\
    &+ \gamma_{n-1}^2 (-1)^{n+1}\sum_{k=1}^{ n/2 } \lambda^{n-2k+1}s_{n-2,k-1} \\
    &= (-1)^{n+1}\sum_{k=0}^{n/2}\lambda^{(n+1)-2k}\left( s_{n-1,k} + \gamma_{n-1}^2 s_{n-2,k-1} \right).
\end{align*}
\indent By the sequence definition:
\begin{equation*}
    = (-1)^{n+1}\sum_{k=0}^{\lfloor (n+1)/2 \rfloor}\lambda^{(n+1)-2k}s_{n,k}.
\end{equation*}

\noindent Case $n$ is odd: When this happens, $\lfloor \frac{n}{2} \rfloor = \frac{n-1}{2}, \lfloor \frac{n-1}{2} \rfloor = \frac{n-1}{2}, \lfloor \frac{n+1}{2} \rfloor = \frac{n+1}{2},$
\begin{align*}
    d_{n+1} &= \lambda(-1)^{n+1}\sum_{k=0}^{(n-1)/2} \lambda^{n-2k}s_{n-1,k} \\
    &+ \gamma_{n-1}^2 (-1)^{n+1}\sum_{k=0}^{(n-1)/2} \lambda^{n-2k-1}s_{n-2,k}.
\end{align*}
Once again reindexing the second sum:
\begin{align*}
    d_{n+1} &= (-1)^{n+1}\sum_{k=0}^{(n-1)/2} \lambda^{n-2k+1}s_{n-1,k} \\
    &+ \gamma_{n-1}^2 (-1)^{n+1}\sum_{k=1}^{(n+1)/2} \lambda^{n-2k+1}s_{n-2,k-1}.
\end{align*}
By adding and subtracting the final term of the sum:
\begin{align*}
    &d_{n+1} = \\
    &(-1)^{n+1}\left(\sum_{k=0}^{(n+1)/2} \left(\lambda^{n-2k+1}s_{n-1,k}\right) - \lambda^0s_{n-1,(n+1)/2}\right) \\
    &+ \gamma_{n-1}^2 (-1)^{n+1}\sum_{k=1}^{(n+1)/2} \lambda^{n-2k+1}s_{n-2,k-1}.
\end{align*}
However, note that $s_{n-1,(n+1)/2} = 0$, thus the $\lambda^0$ term cancels anyway. Thus we have:

\begin{align*}
    d_{n+1} &= (-1)^{n+1} \\
    &\sum_{k=0}^{(n+1)/2} \lambda^{n-2k+1}\left(s_{n-1,k} + \gamma_{n-1}^2 s_{n-2,k-1}\right).
\end{align*}
\indent By the sequence definition:
\begin{equation*}
    d_{n+1} = (-1)^{n+1}\sum_{k=0}^{\lfloor (n+1)/2 \rfloor}\lambda^{(n+1)-2k}s_{n,k}.
\end{equation*}
\end{proof}

\section{Statement and Proof of Lemma \ref{lem:aj}}
\label{app:ajproof}

The following theorem is more general than the subject of this work, and stated accordingly. Also, we remind the reader that a monic polynomial has leading coefficient 1.
\begin{lemma}
\label{lem:aj}
Let $m(x) = \sum_{k=0}^n a_k x^k$ be a degree-$n$ monic polynomial. Then, $r$ is a root of every derivative $m^{(j)}$ for $0 \leq j \leq n-1$ if and only if every coefficient
\begin{align*}
    a_k = \binom{n}{k}(-r)^{n-k}.
\end{align*}
\end{lemma}

\begin{proof}
Starting with the converse, if $a_k=\binom{n}{k}(-r)^{n-k}$, then the Binomial Theorem implies
\begin{align*}
    p(x)=\sum_{k=0}^na_kx^k=\sum_{k=0}^n\binom{n}{k}(-r)^{n-k}x^k=(x-r)^n.
\end{align*}
Clearly, the $j$th derivative has $r$ as a root for $j\in\{0, \ldots, n-1\}$. To establish sufficiency, assume that $r$ is a root of $p^{(n-j)}$ for every $1\le j\le n$. Also, remember from first principles that
\begin{align}
\label{eq:dp}
    p^{(n-j)}(x) = \sum_{k=0}^j\frac{(n-k)!}{(j-k)!}a_{n-k}x^{j-k}.
\end{align}
Of course, $p$ is monic by construction, so $a_n=1$.
Now, suppose strongly that $a_{n-k} = \binom{n}{k}(-r)^k$ for $k<j$.
Then, rearranging Equation \ref{eq:dp} and using the root assumption gives
\begin{align*}
    a_{n-j} &= \frac{p^{(n-j)}(r)}{(n-j)!}-\frac{1}{(n-j)!}\sum_{k=0}^{j-1}\frac{(n-k)!}{(j-k)!}a_{n-k}r^{j-k}, \\
        &= -\frac{r^j}{(n-j)!}\sum_{k=0}^{j-1}\frac{n!}{(j-k)!k!}(-1)^k,    \\
        &=-\binom{n}{j}r^{j}\sum_{k=0}^{j-1}\binom{j}{k}(-1)^k, \\
        &=\binom{n}{j}(-r)^{j}.
\end{align*}
The last step in the preceding equation is due to the Binomial Theorem and the result follows by induction.
\end{proof}

\section{Proof of Lemma \ref{lem:stog}}
\label{app:proofstog}
\begin{proof}
In the case that $n-i > 2k-2$, the values of $\gamma_{n-i}$ may be algebraically solved for in terms of the $s_{n-i,k}$ as:
\begin{equation*}
    \gamma_{n-i}^2 = \frac{s_{n-i,k} - s_{n-(i+1),k}}{s_{n-(i+2),k-1}}.
\end{equation*}
With brute force, and using the proposed expression: $s_{n-i, k}=\frac{\binom{i+k-1}{k}}{\binom{2(i+k-1)}{2k}}\binom{n-i+1}{2k}\lambda_*^{2k}$, the numerator's expression may be solved for as:
\begin{align*}
    &s_{n-i,k} - s_{n-(i+1),k} \\
    &= \lambda_*^{2k}\frac{\binom{i+k-1}{k}\binom{n-i}{2k}}{\binom{2(i+k-1)}{2k}}\left( \frac{2k(n+i)}{(n-i-2k+1)(2i+2k-1)}\right).
\end{align*}
Upon dividing the denominator's expression, $\gamma_{n-i}^2$ becomes:
\begin{align*}
    \gamma_{n-i}^2 &= \lambda_*^2 \frac{\binom{i+k-1}{k}}{\binom{i+k}{k-1}} \frac{\binom{2(i+k)}{2k-2}}{\binom{2(i+k-1)}{2k}} \frac{\binom{n-i}{2k}}{\binom{n-i-1}{2(k-1)}} \\
    &\left( \frac{2k(n+i)}{(n-i-2k+1)(2i+2k-1)}\right).
\end{align*}
Upon simplifying each binomial coefficient, the expression greatly simplifies:
\begin{align*}
    \gamma_{n-i}^2 &= \lambda_*^2\frac{(2i+2k-1)(n-i)(n-i-2k+1)}{2k(2i+1)(2i-1)} \\
    &\left( \frac{2k(n+i)}{(n-i-2k+1)(2i+2k-1)}\right).
\end{align*}
Obvious cancellations may be observed and interestingly enough, all factors involving $k$ cancel out, leaving us with:
\begin{equation*}
    \gamma_{n-i}^2 = \frac{(n-i)(n+i)}{(2i+1)(2i-1)}\lambda_*^2 = \frac{n^2-i^2}{4i^2-1}\lambda_*^2.
\end{equation*}
\end{proof}
\section{Proof of Lemma \ref{lem:differenceeq}}
\label{app:proofdifferenceeq}
\begin{proof}
We establish that Equation \ref{eq:snk} together with Lemma \ref{lem:stog} is a solution to the recurrence in Equation \ref{eq:sjk}.
This is a matter of simple combinatorial arithmetic,
\begin{align*}
    &\gamma^2_{n-i}s_{n-i-2,k-1}+s_{n-i-1,k} \\
    &= \frac{n^2-i^2}{4i^2-1}\frac{\binom{i+k}{k-1}}{\binom{2(i+k)}{2k-2}}\binom{n-i-1}{2k-2}\lambda_*^{2k} \\
    &+ \frac{\binom{i+k}{k}}{\binom{2(i+k)}{2k}}\binom{n-i-1}{2k}\lambda_*^{2k} \\
    &= \frac{\binom{i+k-1}{k}}{\binom{2(i+k-1)}{2k}} \bigg[\frac{n+i}{2(i+k)-1}\binom{n-i}{2k-1} \\
    &+\frac{2i-1}{2(i+k)-1}\binom{n-i}{2k}\bigg]\lambda_*^{2k}.
\end{align*}
Focusing on the second summand, we rearrange to get

\begin{align*}
    &\frac{2i-1}{2(i+k)-1}\binom{n-i}{2k} \\
    &=\left[1-\frac{2k}{2(i+k)-1}\right]\binom{n-i}{2k} \\
    &=\binom{n-i}{2k}-\frac{n-i-2k+1}{2(i+k)-1}\binom{n-i}{2k-1}.
\end{align*}
Substituting this form back in, and simplfiying gives
\begin{align*}
    &\gamma^2_{n-i}s_{n-i-2,k-1}+s_{n-i-1,k} \\
    & = \frac{\binom{i+k-1}{k}}{\binom{2(i+k-1)}{2k}}\left[\binom{n-i}{2k}+\binom{n-i}{2k-1}\right]\lambda_*^{2k} \\
    &=\frac{\binom{i+k-1}{k}}{\binom{2(i+k-1)}{2k}}\binom{n-i+1}{2k}\lambda_*^{2k} = s_{n-i,k}.
\end{align*}
Of course, if $n-i\le 2k-2$ then $\binom{n-i+1}{2k}=0$; otherwise, we can easily verify that $s_{n-i,0}=1$.
So, Equation \ref{eq:snk} is a solution to the recurrence in Equation \ref{eq:sjk} under the given boundary conditions.

\end{proof}

\section{HOLD dynamics used for the toy dataset}

This appendix summarizes the calculations of $\exp(\bm{F}t)$ for the traditional HOLD algorithm for hand-picked eigenvalues (denoted by $\lambda$). Note that these calculations pertain to HOLD, not the proposed HOLD++ method.

\subsection{2nd order HOLD}

%\[F = \begin{pmatrix}
%    0 & 1 \\
%    -1 & -\xi
%\end{pmatrix}\]

%\[\det\begin{pmatrix}
%    -\lambda & 1 \\
%    -1 & -\xi-\lambda
%\end{pmatrix} = \lambda(\lambda + \xi) + 1\]
%\[= \lambda^2 + \xi\lambda + 1 = (\lambda+\sqrt{2})(\lambda + \frac{\sqrt{2}}{2})\]
%\[\lambda = -\sqrt{2}, -\frac{\sqrt{2}}{2} \to \xi = \frac{3\sqrt{2}}{2}\]
%\[F = \begin{pmatrix}
%    0 & 1 \\
%    -1 & -\frac{3}{\sqrt{2}}
%\end{pmatrix}\]

\[\bm{F} = \begin{pmatrix}
    0 & 1 \\
    -1 & -\xi
\end{pmatrix}, \xi = \frac{3\sqrt{2}}{2} \to \lambda = -\sqrt{2}, -\frac{\sqrt{2}}{2}\]

\[\bm{P}_1 = \bm{I}, \bm{P}_2 = (\bm{F} + \sqrt{2}\bm{I})\bm{P}_1\]
\begin{align*}
    %r'_1(t) &= -\sqrt{2}r_1(t), \quad & r_1(0) = 1 \\
    %r'_2(t) &= -\frac{1}{\sqrt{2}}r_2(t) + r_1(t), \quad & r_2(0) = 0 \\
    r_1(t) &= \exp(-\sqrt{2}t) \\
    r_2(t) &= \sqrt{2} \left(\exp(-\frac{1}{\sqrt{2}}t) - \exp(-\sqrt{2}t)\right)
\end{align*}
\[\exp(\bm{F}t) = \sum_{i=1}^2 r_i(t)\bm{P}_i\]

\subsection{3rd order HOLD}
This was already done in HOLD's previous works, but it is attached for completeness:

%\[F = \begin{pmatrix}
%    0 & 1 & 0 \\
%    -1 & 0 & \gamma_2 \\
%    0 & -\gamma_2 & -\xi \\
%\end{pmatrix}\]

%\[F-\lambda I = \begin{pmatrix}
%    -\lambda & 1 & 0 \\
%    -1 & -\lambda & \gamma_2 \\
%    0 & -\gamma_2 & -\xi-\lambda \\
%\end{pmatrix}\]

%\begin{align*}
%    \det(F - \lambda I) &= (-\xi-\lambda)(\lambda^2 + 1) - \lambda\gamma_2^2 \\
%    &= -(\lambda^3 + \xi\lambda^2 + \lambda(1 + \gamma_2^2) + \xi) \\
%    &= -(\lambda+1)(\lambda+2)(\lambda+3)
%\end{align*}
%\[\gamma_2 = \sqrt{10}, \xi = 6\]
%\[P_1 = I, \quad P_2 = (F + 3I)P_1, \quad P_3 = (F+2I)P_2\]
%\begin{align*}
%    r'_1(t) &= -3r_1(t), \quad & r_1(0) = 1 \\
%    r'_2(t) &= -2r_2(t) + r_1(t), \quad & r_2(0) = 0 \\
%    r'_3(t) &= -r_3(t) + r_2(t), \quad & r_3(0) = 0 \\    
%\end{align*}

%\[\bm{F} = \begin{pmatrix}
%    0 & 1 & 0 \\
%    -1 & 0 & \gamma_2 \\
%    0 & -\gamma_2 & -\xi \\
%\end{pmatrix}, \gamma_2 = \sqrt{10}, \xi = 6, \to\]
\[\gamma_2 = \sqrt{10}, \xi = 6, \to\]
\[\lambda=-3,-2,-1\]
\[\bm{P}_1 = \bm{I}, \quad \bm{P}_2 = (\bm{F} + 3\bm{I})\bm{P}_1, \quad \bm{P}_3 = (\bm{F}+2\bm{I})\bm{P}_2\]
\begin{align*}
    r_1(t) &= \exp(-3t) \\
    r_2(t) &= \exp(-2t) - \exp(-3t) \\
    r_3(t) &= \frac{1}{2}\left(\exp(-t) + \exp(-3t) - 2\exp(-2t) \right)
\end{align*}
\[\exp(\bm{F}t) = \sum_{i=1}^3 r_i(t)\bm{P}_i\]

\subsection{4th order HOLD}
\[\gamma_1 = \frac{2}{\sqrt{5}}, \gamma_2 = \sqrt{\frac{21}{5}}, \gamma_3 = \sqrt{30}, \xi=10 \to \]
\[\lambda = -\frac{\sqrt{5}}{2}, -\sqrt{5}, -\frac{3\sqrt{5}}{2}, -2\sqrt{5}\]
\[\bm{P}_1 = \bm{I}, \bm{P}_2 = (\bm{F} + 2\sqrt{5}\bm{I})\bm{P}_1,\]
\[\bm{P}_3 = (\bm{F} + \frac{3\sqrt{5}}{2}\bm{I})\bm{P}_2, \bm{P}_4 = (\bm{F} + \sqrt{5}\bm{I})\bm{P}_3\]

\begin{align*}
    r_1(t) &= \exp(-2\sqrt{5}t), \\
    r_2(t) &= \frac{2\sqrt{5}}{5}\left(\exp(-\frac{3\sqrt{5}}{2}t) - \exp(-2\sqrt{5}t)\right), \\
    r_3(t) &= \frac{2}{5}\bigg(\exp(-\sqrt{5}t) + \exp(-2\sqrt{5}t) \\
    &- 2\exp(-\frac{3\sqrt{5}}{2}t) \bigg), \\
    r_4(t) &= \frac{4\sqrt{5}}{75}\bigg(\exp(-\frac{\sqrt{5}}{2}t)+3\exp(-\frac{3\sqrt{5}}{2}t) \\
    &-3\exp(-\sqrt{5}t) - \exp(-2\sqrt{5}t)\bigg)
\end{align*}
\[\exp(\bm{F}t) = \sum_{i=1}^4 r_i(t)\bm{P}_i\]

\section{CelebA-HQ $256 \times 256$ samples}

This appendix contains Figure \ref{fig:celebahq}, the visual samples from CelebA-HQ $256 \times 256$, in the interest of space.

\begin{figure}
    \centering
    \includegraphics[width=0.7\linewidth]{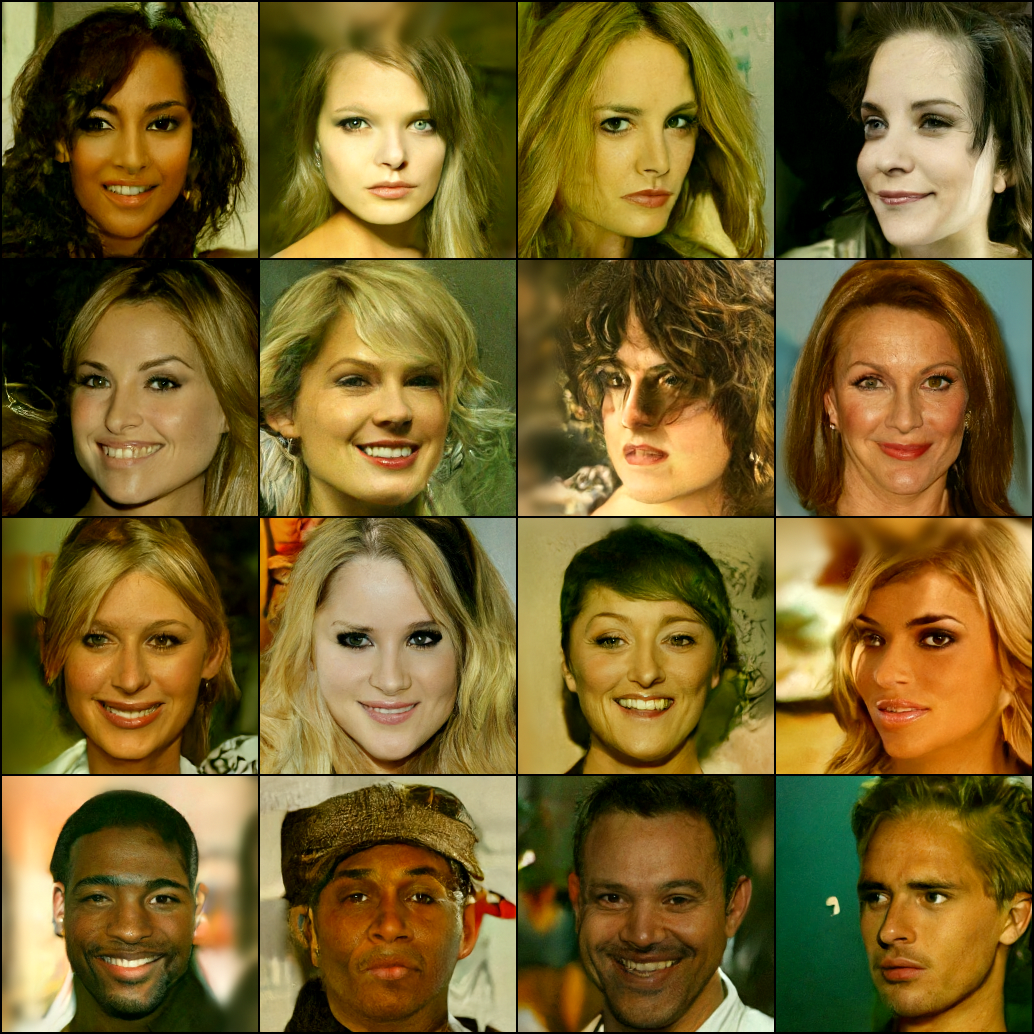}
    \caption{Generated samples from CelebA-HQ for model order 3 without cherry picking}
    \label{fig:celebahq}
\end{figure}

\end{document}